\documentclass[DIV=calc,10pt,twocolumn]{article} 
\usepackage{simpleConference}
\usepackage{times}
\usepackage{graphicx,float}
\usepackage{amssymb}
\usepackage{url}
\usepackage{caption,lipsum}
\usepackage{abstract}

\usepackage{amsmath}
\usepackage{amsthm}
\usepackage{booktabs} 
\usepackage{wrapfig}
\usepackage{tabularx}

\newtheorem{theorem}{Theorem}

\newtheorem{lemma}{Lemma}

\newtheorem{proposition}{Proposition}

\newcommand\question[1]

\newcommand{\abs}[1]{\left\vert#1\right\vert}

\newcommand{\set}[1]{\left\{#1\right\}}
\newcommand{\parr}[1]{\left (#1\right )}

\newcommand{\Real}{\mathbb R}

\newcommand{\ones}{\mathbf{1}}

\newcommand{\E}{\mathcal  E}
\newcommand{\R}{\mathcal  R}

\newcommand{\too}{\rightarrow}

\newcommand{\OO}{\mathcal{O}}

\newcommand{\area}{\textrm{area}}

\newcommand{\eg}{{\it e.g.}}
\newcommand{\ie}{{\it i.e.}}

\begin{document}

\title{Point Convolutional Neural Networks by Extension Operators}

\author{Matan Atzmon\thanks{equal contribution} \and Haggai Maron\footnotemark[1] \and Yaron Lipman\\}
\affiliation{Weizmann Institute of Science}



\twocolumn[{%
	\renewcommand\twocolumn[1][]{#1}%
	\maketitle
	\begin{center}
		\vspace{-0.9cm}
		\includegraphics[width=7.0in]{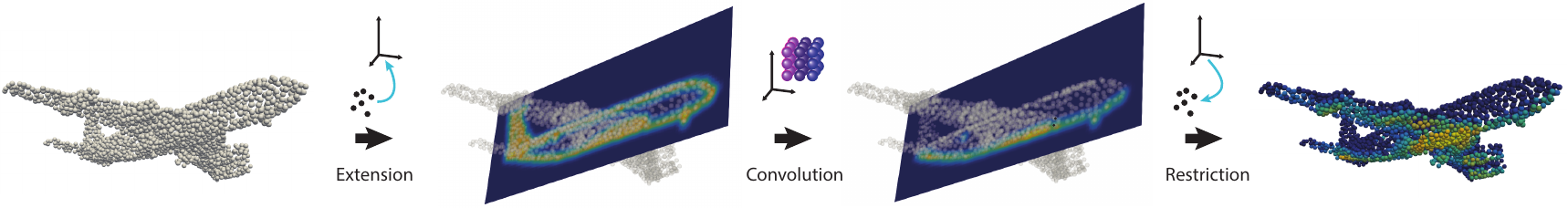}	\captionof{figure}{A new framework for applying convolution to functions defined over point clouds: First, a function over the point cloud (in this case the constant one) is \emph{extended} to a continuous volumetric function over the ambient space; second, a continuous volumetric \emph{convolution} is applied to this function (without any discretization or approximation); and lastly, the result is \emph{restricted} back to the point cloud.  }\label{fig:teaser}
	\end{center}%
}]
\saythanks

\subsection*{\centering Abstract} 

\textit{This paper presents Point Convolutional Neural Networks (PCNN): a novel framework for applying convolutional neural networks to point clouds. The framework consists of two operators: extension and restriction, mapping point cloud functions to volumetric functions and vise-versa. A point cloud convolution is defined by pull-back of the Euclidean volumetric convolution via an extension-restriction mechanism.}

\textit{The point cloud convolution is computationally efficient, invariant to the order of points in the point cloud, robust to different samplings and varying densities, and translation invariant, that is the same convolution kernel is used at all points. PCNN generalizes image CNNs and allows readily adapting their architectures to the point cloud setting.}

\textit{Evaluation of PCNN on three central point cloud learning benchmarks 	convincingly outperform competing point cloud learning methods, and the vast majority of methods working with more informative shape representations such as surfaces and/or normals.}


\maketitle

\vspace{-0.3cm}
\section{Introduction}

The huge success of deep learning in image analysis motivates researchers to generalize deep learning techniques to work on 3D shapes. Differently from images, 3D data has several popular representation, most notably surface meshes and points clouds. Surface-based methods exploit connectivity information for 3D deep learning based on rendering  \cite{su2015multi}, local and global parameterization \cite{masci2015geodesic,sinha2016deep,maron2017convolutional}, or spectral properties \cite{yi2016syncspeccnn}. Point cloud methods rely mostly on points' locations in three-dimensional space and need to implicitly infer how the points are connected to form the underlying shape.

The goal of this paper is to introduce Point Cloud Convolutional Neural Networks (PCNN) generalizing deep learning techniques, and in particular Convolutional Neural Networks (CNN) \cite{krizhevsky2012imagenet}, to point clouds. As a point cloud $X\subset \Real^3$ is merely an approximation to some underlying shape $S$, the main challenges in building point cloud networks are to achieve: (i) Invariance to the order of points supplied in $X$; (ii) Robustness to sampling density and distribution of $X$ in $S$; and (iii) Translation invariance of the convolution operator (\ie, same convolution kernel is used at all points) .

Invariance to point order in $X$ was previously tackled in \cite{qi2016pointnet,ravanbakhsh2016deep,qi2017pointnet++, zaheer2017deep} by designing networks that are composition of euuquivariant layers (\ie, commute with permutations) and a final symmetric layer (\ie, invariant to permutations). As shown in \cite{ravanbakhsh2016deep}, any linear equivariant layer is a combination of scaled identity and constant linear operator and therefore missing many of the degrees of freedom existing in standard linear layers such as fully connected and even convolutional.

Volumetric grid methods \cite{wu20153d,maturana2015voxnet,qi2016volumetric,riegler2016octnet} use 3D occupancy grid to deal with the point order in $X$ and provide translation invariance of the convolution operator. However, they quantize the point cloud to a 3D grid, usually producing a crude approximation to the underlying shape (\ie, piecewise constant on voxels) and are confined to a fixed 3D grid structure.

Our approach toward these challenges is to define CNN on a point cloud $X$ using a pair of operators we call \emph{extension} $\E_X$ and \emph{restriction} $\R_X$. The extension operator maps functions defined over the point cloud $X$ to volumetric functions (\ie, functions defined over the entire ambient space $\Real^3$), where the restriction operator does the inverse action. Using $\E_X, \R_X$ we can translate operators such as Euclidean volumetric convolution to point clouds, see Figure \ref{fig:teaser}. In a nutshell, if $O$ is an operator on volumetric functions then its restriction to the point cloud $X$ would be
\begin{equation}\label{e:key}
	O_X = \R_X \circ O \circ \E_X.
\end{equation}
We take $\E_X$ to be a Radial Basis Function (RBF) approximation operator,
and $\R_X$ to be a sampling operator, \ie, sample a volumetric function at the points in $X$. As $O$ we take continuous volumetric convolution operators with general kernels $\kappa$ represented in the RBF basis as-well. In turn \eqref{e:key} is calculated using a sparse linear tensor combining the learnable kernel weights $k$, function values over the point cloud $X$, and a tensor connecting the two, defined directly from the point cloud $X$.

Since our choice of $\E_X$ is invariant to point order in $X$, and $\R_X$ is an equivariant operator (w.r.t.~$X$) we get that $O_X$ in \eqref{e:key} is equivariant. This construction leads to new equivariant layers, in particular convolutions, with more degrees of freedom compared to \cite{ravanbakhsh2016deep,qi2016pointnet, zaheer2017deep}.
The second challenge of robustness to sampling density and distribution is addressed by the \emph{approximation power} of the extension operator $\E_X$. Given a continuous function defined over a smooth surface, $f:S\too\Real$, we show that the extension of its restriction to $X$ approximates the restriction of $f$ to $S$, namely $$\E_X\circ \R_X[f]\approx f\Big\vert_{S}.$$ This means that two different samplings $X,X'\subset S$ of the same surface function are extended to the \emph{same volumetric function}, up to an approximation error. In particular, we show that extending the simplest, constant one function over the point cloud, $\E_X[\ones]$, approximates the indicator function of the surface $S$, while the gradient, $\nabla \E_X[\ones]$, approximates the mean curvature normal field over the surface. Then, the translation invariance and robustness of our convolution operator naturally follows from the fact that the volumetric convolution is translation invariant and the extension operator is robust.

PCNN provides a flexible framework for adapting standard image-based CNNs to the point cloud setting, while maintaining data only over the point cloud on the one hand, and learning convolution kernels robust to sampling on the other. We have tested our PCNN framework on standard classification, segmentation and normal estimation datasets where PCNN outperformed all other point cloud methods and the vast majority of other methods that use more informative shape representations such as surface connectivity.

\section{Previous Work}

We review different aspects of geometric deep learning with a focus on the point cloud setting. For a more comprehensive survey on geometric deep learning we refer the reader to \cite{bronstein2017geometric}.

\paragraph{Deep learning on point clouds.}
PointNet \cite{qi2016pointnet} pioneered deep learning for point clouds with a Siamese, per-point network composed with a symmetric max operator that guarantees invariance to the points' order. PointNet was proven to be a universal approximator (\ie, can approximate arbitrary continuous functions over point clouds). A follow up work \cite{qi2017pointnet++} suggests a hierarchical application of the PointNet model to different subsets of the point cloud; this allows capturing structure at different resolutions when applied with a suitable aggregation mechanism. In \cite{guerrero2017pcpnet} the PointNet model is used to predict local shape properties from point clouds. In a related work \cite{ravanbakhsh2016deep,zaheer2017deep} suggest to approximate set function, with equivariant layers composed with a symmetric function such as max.
Most related to our work is the recent work of \cite{klokov2017escape} that suggested to generalize convolutional networks to point clouds by defining convolutions directly on kd-trees built out of the point clouds \cite{bentley1975multidimensional}, and \cite{schutt2017moleculenet} that suggested a convolutional architecture for modeling quantum interactions in molecules represented as point clouds, where convolutions are defined by multiplication with continuous filters. The main difference to our work is that we define the convolution of a point cloud function using an exact volumetric convolution with an extended version of the function. The approximation properties of the extended function facilitate a robust convolution on point clouds.

\paragraph{Volumetric methods.}
Another strategy is to generate a tensor volumetric representation of the shape restricted to a regular grid (\eg, by using occupancy indicators, or a distance function) \cite{wu20153d,maturana2015voxnet,qi2016volumetric}. The main limitation of these methods is the approximation quality of the underlying shape due to the low resolution enforced by the three dimensional grid structure. To overcome this limitation a few methods suggested to use sparse three dimensional data structures such as octrees \cite{wang2017cnn,riegler2016octnet}.
Our work can be seen as a generalization of these volumetric methods in that it allows replacing the grid cell's indicator functions as the basis for the network's functions and convolution kernels with more general basis functions (\eg, radial basis functions).

\paragraph{Deep learning on Graphs.}
Shapes can be represented as graphs, namely points with neighboring relations. In spectral deep learning the convolution is being replaced by a diagonal operator in the graph-Laplacian eigenbasis \cite{bruna2013spectral, defferrard2016convolutional, henaff2015deep}. The main limitation of these methods in the context of geometric deep learning is that different graphs have different spectral bases and finding correspondences between the bases or common bases is challenging. This problem was recently targeted by \cite{yi2016syncspeccnn} using the functional map framework.

\paragraph{Deep learning on surfaces.}
Other approaches to geometric deep learning work with triangular meshes that posses also connectivity and normal information, in addition to the point locations. One class of methods use rendering and 2D projections to reduce the problem to the image setting \cite{su2015multi,kalogerakis2016multiviewsegment}.
Another line of works uses local surface representations  \cite{masci2015geodesic,boscaini2016learning,monti2016geometric} or global parameterizations of surfaces \cite{sinha2016deep,maron2017convolutional} for reducing functions on surfaces to the planar domain or for defining convolution operators directly over the surfaces.

\paragraph{RBF networks.} RBF networks are a type of neural networks that use RBF functions as an activation layer, see \cite{orr1996introduction,orr1999recent}. This model was first introduced in \cite{broomhead1988radial}, and was used, among other things, for function approximation and time series prediction.
Usually, these networks have three layers and their output is a linear combination of radial basis functions. Under mild conditions this model can be shown to be a universal approximator of functions defined on compact subsets of $\Real^d$ \cite{park1991universal}. Our use of RBFs is quite different: RBFs are used in our extension operator solely for the purpose of defining point cloud operators, whereas the ReLU is used as an activation.


\section{Method}
\paragraph{Notations.}
We will use tensor (\ie, multidimensional arrays) notation,
\eg, $a\in \Real^{I\times I \times J \times L \times M}$. Indexing a particular entry is done using corresponding lower-case letters, $a_{ii'jlm}$, where $1\leq i,i' \leq I$, $1\leq j \leq J$, etc. When summing tensors $c=\sum_{ijl}a_{ii'jlm}b_{ijl}$, where $b\in\Real^{I\times J\times L}$ the dimensions of the result tensor $c$ are defined by the free indices, in this case $c=c_{i'm}\in\Real^{I\times M}$.

\paragraph{Goal.}
Our goal is to define convolutional neural networks on point clouds $X=\set{x_i}_{i=1}^I\in \Real^{I\times 3}$. Our approach to defining point cloud convolution is to extend functions on point clouds to volumetric functions, perform standard Euclidean convolution on these functions and sample them back on the point cloud.


We define an extension operator
\begin{equation}
	\E_X : \Real^{I\times J} \too C(\Real^3, \Real^J),
\end{equation}
where $\Real^{I\times J}$ represents the collection of functions $f:X\too \Real^J$, and $C(\Real^3,\Real^J)$ volumetric functions $\psi:\Real^3 \too \Real^J$. Together with the extension operator we define the restriction operator
\begin{equation}
	\R_x : C(\Real^3, \Real^M) \too  \Real^{I\times M}.
\end{equation}
Given a convolution operator $O:C(\Real^3,\Real^J)\too C(\Real^3, \Real^M)$ we adapt $O$ to the point cloud $X$ via \eqref{e:key}. We will show that a proper selection of such point cloud convolution operators possess the following desirable properties:
\begin{enumerate}
	\item \emph{Efficiency}: $O_X$ is computationally efficient.
	\item \emph{Invariance}: $O_X$ is indifferent to the order of points in $X$, that is, $O_X$ is equivariant.
	\item \emph{Robustness}: Assuming $X\subset S$ is a sampling of an underlying surface $S$, and $f\in C(S,\Real)$, then $\E_X\circ \R_X [f]\in C(\Real^3,\Real)$ approximates $f$ when sampled over $S$ and decays to zero away from $S$. In particular $\E_X[\ones]$ approximates the volumetric indicator function of $S$, where $\ones\in \Real^{I\times 1}$ is the vector of all ones; $\nabla \E_X[\ones]$ approximates the mean curvature normal field over $S$. The approximation property in particular implies that if $X,X^*\subset S$ are different samples of $S$ then $O_X \approx O_{X*}$.
	\item \emph{Translation invariance}: $O_X$ is translation invariant, defined by a stationary (\ie, location independent) kernel.
\end{enumerate}

In the next paragraphs we define these operators and show how they are used in defining the main building blocks of PCNN, namely: convolution, pooling and upsampling. We discuss the above theoretical properties in Section \ref{s:preperties}.
\vspace{-0.2cm}
\subsection*{Extension operator}

The extension operator $\E_X:\Real^{I\times J}\too C(\Real^3,\Real^J)$ is defined as an operator of the form,
\begin{equation}\label{e:extension}
\E_X[f](x)= \sum_i f_{ij} \ell_i(x),
\end{equation}

where $f\in \Real^{I\times J}$, and $\ell_i\in C(\Real^3,\Real)$ can be thought of as basis functions defined per evaluation point $x$.
One important family of bases are the Radial Basis Functions (RBF), that were proven to be useful for surface representation\cite{berger2017survey,carr2001reconstruction}. For example, one can consider interpolating bases (\ie, satisfying $\ell_i(x_{i'}) = \delta_{ii'}$) made out of an RBF $\Phi:\Real_+\too \Real$. Unfortunately, computing \eqref{e:extension} in this case amounts to solving a dense linear system of size $I\times I$. Furthermore, it suffers from bad condition number as the number of points is increased \cite{wendland2004scattered}. In this paper we will advocate a novel approximation scheme of the form
\begin{equation}\label{e:ell}
	\ell_i(x)=c\omega_i \Phi(|x-x_i|),
\end{equation}
where $c$ is a constant depending on the RBF $\Phi$ and $\omega_i$ can be thought of the amount of shape area corresponding to point $x_i$. A practical choice of $\omega_i$ is
\begin{equation}\label{e:practical_choice_of_omega}
	\omega_i = \frac{1}{c\sum_{i'} \Phi(|x_{i'}-x_i|)}.
\end{equation}
Note that although this choice resembles the Nadaraya-Watson kernel estimator \cite{nadaraya1964estimating}, it is in fact different as the denominator is independent of the approximation point $x$; this property will be useful for the closed-form calculation of the convolution operator.

As we prove in Section \ref{s:preperties}, the point cloud convolution operator, $O_X$, defined using the extension operator, \eqref{e:extension}-\eqref{e:ell}, satisfies the properties (1)-(4) listed above, making it suitable for deep learning on point clouds. In fact, as we show in Section \ref{s:preperties}, robustness is the result of the extension operator $\E_X$ approximating a continuous, sampling independent operator over the underlying surface $S$ denoted $\E_S$. This continuous operator applied to a function $f$,  $\E_S[f]$, is proved to approximate the restriction of $f$ to the surface $S$.

Figure \ref{fig:extension_op} demonstrates the robustness of our extension operator $\E_X$; applying it to the constant one function, evaluated on three different sampling densities of the same shape, results in approximately the same shape.

\begin{figure}
	\includegraphics[width=\columnwidth]{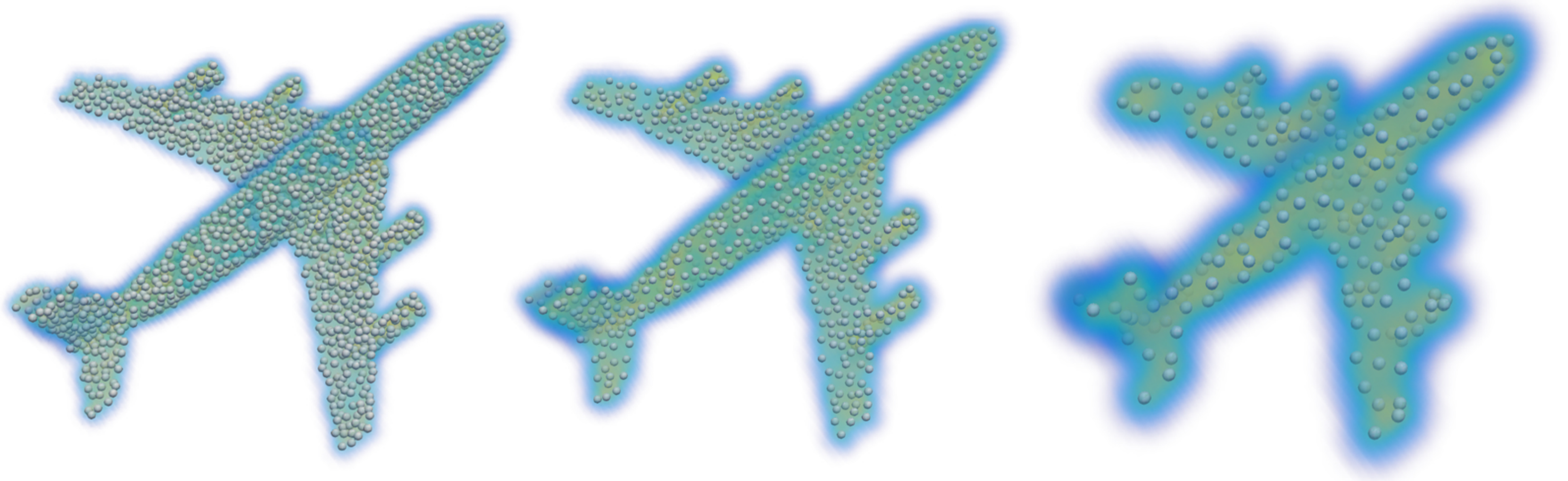}
	\caption {Applying the extension operator to the constant $\ones$ function over three airplane point clouds in different sampling densities: $2048$, $1024$ and $256$ points. Note how the extended functions resemble the airplane indicator function, and hence similar to each other.\vspace{-15pt}}\label{fig:extension_op}
	\end{figure}

\subsection*{Kernel model}
We consider a continuous convolution operator $O:C(\Real^3,\Real^J)\too C(\Real^3,\Real^M)$ applied to vector valued function $\psi\in C(\Real^3,\Real^J)$,
\begin{equation}\label{e:cont_conv}
	O[\psi](x) = \psi * \kappa \,(x) = \int_{\Real^3} \sum_j \psi_j(y)\, \kappa_{jm}(x-y) \, dy,
	\vspace{-0.1cm}
\end{equation}

where $\kappa \in C(\Real^3,\Real^{J\times M})$ is the convolution kernel that is also represented in the same RBF basis:
\begin{equation}\label{e:kernel}
\kappa_{jm}(z)=\sum_l k_{ljm}\Phi(|z-y_l|),
\end{equation}
where with a slight abuse of notation we denote by $k\in\Real^{L\times J \times M}$ the tensor representing the continuous kernel in the RBF basis. Note, that $k$ represents the network's learnable parameters, and has similar dimensions to the convolution parameters in the image case (\ie, spatial dimensions $\times$ input channels $\times$ output channels). 
\begin{wraptable}[8]{r}{0.04\columnwidth}
	\vspace{12pt}\hspace{-0pt}
	\includegraphics[width=0.07\columnwidth]{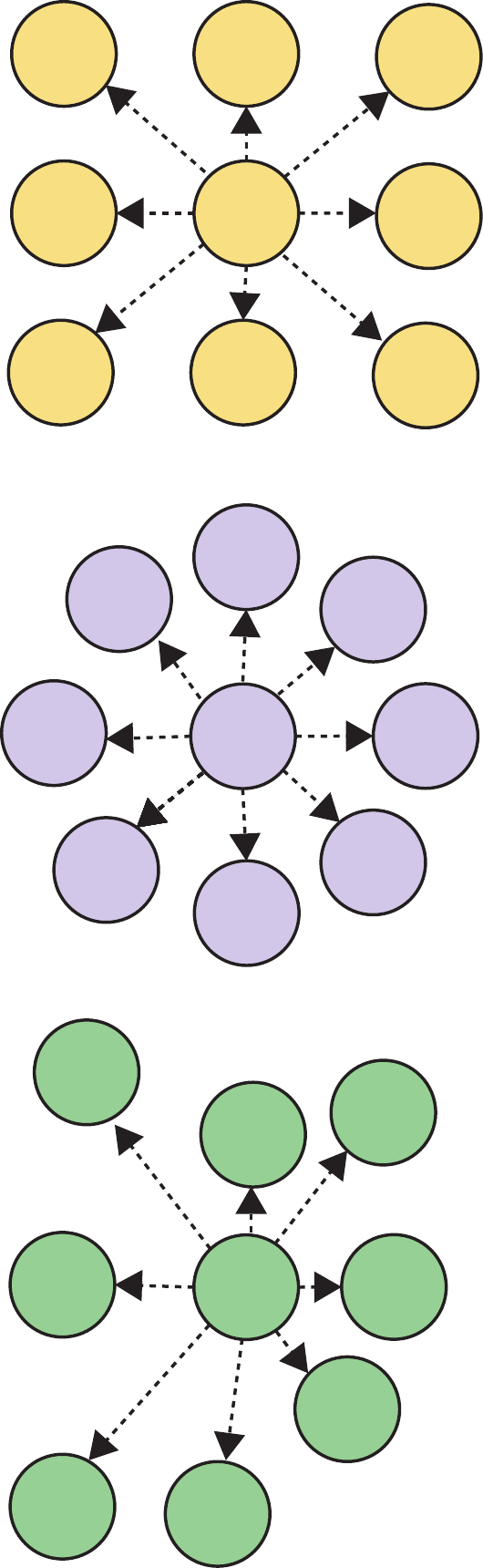}
	\vspace{-0cm}
\end{wraptable}

The translations  $\set{y_l}_{l=1}^L\subset \Real^3$ are also a degree of freedom and can be chosen to generate a regular $3\times 3 \times 3$ grid or any other point distribution such as spherical equispaced points. The translations can be predefined by the user or learned (with some similarly to \cite{dai2017deformable}). See inset for illustration of some possible translations.

\subsection*{Restriction operator}
Our restriction operator $\R_X:C(\Real^3,\Real^J)\too \Real^{I\times J}$ is the sampling operator over the point cloud $X$,
\begin{equation}\label{e:restriction}
	\R_X[\psi] = \psi_j(x_i),
\end{equation}
where $\psi\in C(\Real^3,\Real^J)$. Note that $\R_X[\psi]\in \Real^{I\times J}$.

\subsection*{Sparse extrinsic convolution}
We want to compute the convolution operator $O_X:\Real^{I\times J}\too \Real^{I\times M}$ restricted to the point cloud $X$ as defined in \eqref{e:key} with the convolution operator $O$ from \eqref{e:cont_conv}. First, we compute $\E_X[f]*k$
\begin{align*}
	\E_X[f]*k\, (x) & =  c \sum_{ijl} f_{ij} k_{ljm} w_i \int_{\Real^3}  \Phi(|y-x_{i}|)\Phi(|x-y-y_l|)\,dy
\end{align*}

Applying our restriction operator finally gives our point cloud convolution operator:
\begin{equation}\label{e:point_cloud_convolution}
	 \boxed{O_X[f] = c\sum_{ijl} f_{ij} k_{ljm} w_{i} q_{ii'l},}
\end{equation}
where $q=q(X)\in \Real^{I\times I \times L}$ is the tensor defined by
\begin{equation}\label{e:pair_conv}
	q_{ii'l}=\int_{\Real^3}  \Phi(|y-x_{i}|)\Phi(|x_{i'}-y-y_l|)\,dy.
\end{equation}
Note that $O_X[f] \in\Real^{I\times M}$, as desired. Equation \eqref{e:point_cloud_convolution} shows that the convolution's weights $k_{ljm}$ are applied to the data $f_{ij}$ using point cloud-dependent weights $w$, $q$ that can be seen as "translators" of $k$ to the point cloud geometry $X$.  Figure \ref{fig:convolutionComputationalFlow} illustrates the computational flow of the convolution operator.
\begin{figure}
	\includegraphics[width=\columnwidth]{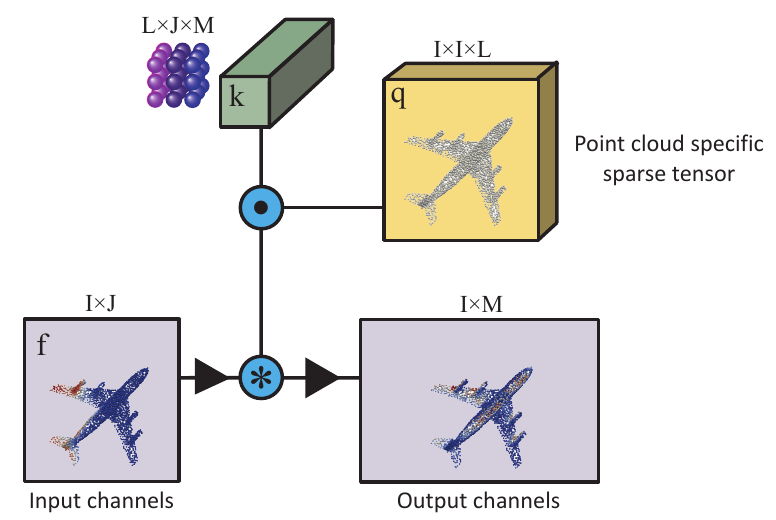}
	\caption {Point cloud convolution operator, computational flow. }\label{fig:convolutionComputationalFlow}
\end{figure}




\subsection*{Choice of RBF}
Our choice of radial basis function $\Phi$ stems from two desired properties: First, we want the extension operator \eqref{e:extension} to have approximation properties; second, we want the computation of the convolution of a pair of RBFs in \eqref{e:pair_conv} to have an efficient closed-form solution. A natural choice satisfying these requirements is the Gaussian:
\begin{equation}\label{e:gaussian}
\Phi_{\sigma}(r)=\exp\parr{-\frac{r^2}{2\sigma^2}}
\end{equation}
To compute the tensor $q\in \Real^{I\times I \times L}$ in \eqref{e:pair_conv} we make use of the following convolution rule for Gaussians, proved in Appendix \ref{app:proofs} for completeness:
\begin{lemma}\label{lem:gaussians}
	Let $\Phi$ denote the Gaussian as in \eqref{e:gaussian}. Then,
	\begin{equation}\label{e:conv_rule}
	\Phi_\alpha (|\cdot-a|)*\Phi_\beta (|\cdot-b|) \propto \Phi_{\gamma}(|\cdot-a-b|),
	\end{equation}
	where $\gamma = \sqrt{\alpha^2+\beta^2}$.

\end{lemma}


\subsection{Up-sampling and pooling}
Aside from convolutions, there are two operators that perform spatially and need to be defined for point clouds: up-sampling $U_{X,X^*}:\Real^{I\times J} \too \Real^{I^*\times J}$, and pooling $P_{X,X^*}:\Real^{I\times J} \too \Real^{I^*\times J}$, where $X^*\subset \Real^3$ is superset of $X$ (\ie, $I^*>I$) in the upsampling case and subset of $X$ (\ie, $I^*<I$) in the pooling case. The upsample operator is defined by
\begin{equation}\label{e:upsample}
	U_{X,X^*}[f]=\R_{X^*}\circ \E_X [f].
\end{equation}
Pooling does not require the extension/restriction operators and (similarly to \cite{qi2017pointnet++}) is defined by
\begin{equation}
	P_{X,X^*}[f](x^*_i) = \max_{i\in \mathcal{V}_{i^*}} f_{ij},
\end{equation}
where $\mathcal{V}_{i^*}\subset \set{1,2,\ldots,I}$ denotes the set of indices of points in $X$ that are closer in Euclidean distance to $x^*_i$ than any other point in $X^*$.
The point cloud $X^* \subset X$ in the next layer is calculated using farthest point sampling of the input point cloud $X$.

Lastly, similarly to \cite{karras2017progressive} we implement \emph{deconvolution} layers by an upsampling layer followed by a regular convolution layer.

\section{Properties }\label{s:preperties}

In this section we discuss the properties of the point cloud operators we have defined above.

\subsection{Invariance and equivariance}
Given a function $f\in \Real^{I \times J}$ on a point cloud $X\in\Real^{I\times 3}$, an equivariant layer $L: \Real^{I\times J} \too \Real^{I\times M}$ satisfies $$L(\pi f) = \pi L(f),$$ where $\pi\in\Pi_I\subset \Real^{I\times I}$ is an arbitrary permutation matrix. Equivariant layers have been suggested in \cite{qi2016pointnet,ravanbakhsh2016deep,zaheer2017deep} to learn data in the form of point clouds (or sets in general). The key idea is that equivariant layers can be used to represent a set function $F:2^{\Real^3}\too \Real$. Indeed, a set function restricted to sets of fixed size (say, $I$) can be represented as a symmetric function (\ie, invariant to the order of its arguments). A rich class of symmetric functions can be built by composing equivariant layers and a final symmetric layer.

The equivariance of our point cloud operators $O_X$ stems from the invariance property of the extension operator and equivariance property of the restriction operator. We will next show these properties.
\begin{lemma}
	The extension operators defined in \eqref{e:extension} is invariant to permutations, \ie, $\E_{\pi X}[f]=\E_X[f]$, for all permutations $\pi\in \Pi_I$. The restriction operator \eqref{e:restriction} is equivariant to permutations, $\R_{\pi X}[\psi]=\pi \R_X[\psi]$, for all $\pi\in\Pi_I$.
\end{lemma}

\begin{proof}
	The properties follow from the definitions of the operators. 	
	$$\E_{\pi(X)}[f]=\sum_i f_{\pi(i)j} \ell_{\pi(i)}= \sum_i f_{ij} \ell_{i}=\E_X[f], $$
	and  $$\R_X[\psi]=\psi_j(x_{\pi(i)}) = \Pi \R_X[\psi].$$
\end{proof}
A consequence of this lemma is that any convolution $O$ acting on volumetric functions in $\Real^3$ translates to an equivariant operator $O_X$,
\begin{theorem}\label{thm:OX_equivariant}
	Let $O:C(\Real^3,\Real^J)\too C(\Real^3,\Real^M)$ be a volumetric function operator. Then $O_X:\Real^{I\times J}\too \Real^{I\times M}$ defined by \eqref{e:key} is equivariant. Namely, $$O_{\pi X}[f]=\pi O_X [f].$$
\end{theorem}
Theorem \ref{thm:OX_equivariant} applies in particular to convolutions \eqref{e:cont_conv}, and therefore our point cloud convolutions are all equivariant by construction. Note that this model provides "data-dependent" equivariant operator that are more general than those suggest in \cite{ravanbakhsh2016deep,zaheer2017deep}.

\subsection{Robustness}

\paragraph{Overview.} Robustness is the key property that allows applying the same convolution kernel to functions over different irregular point clouds. The key idea is to make the extension operator produce approximately the same volumetric function when applied to different samplings of the same underlying shape function. To make things concrete, let $X\in\Real^{I\times 3}$, $X^*\in\Real^{I^*\times 3}$ be two different point clouds samples of a compact smooth surface $S\subset \Real^3$. Let $f\in C(S,\Real^J)$ be some function over $S$ and $\R_X[f], \R_{X^*}[f]$ its sampling on the points clouds $X,X^*$, respectively.

We will show the following:
\begin{enumerate}
	\item We introduce a continuous extension operator $E_S$  from surface functions to volumetric functions. We show that $E_S$ has several favorable properties.
	\item We show that (under mild assumptions) our extension operator $\E_X$, defined in \eqref{e:extension}-\eqref{e:ell} converges to $\E_S$,
	\begin{equation}\label{e:EX_approx_ES}
	\E_X \circ \R_X [f] \approx \E_S[f].
	\end{equation}	
	\item We deduce that (under mild assumptions) the properties of $\E_S$ are inherited by $\E_X$ and in particular we have:
	\begin{equation}\label{e:consistency}
	\E_X \circ \R_X[f]\approx \E_{X^*} \circ \R_{X^*}[f].
	\end{equation}
\end{enumerate}

\paragraph{Continuous extension operator.}
We define $\E_S: C(S,\Real^J) \too C(\Real^3,\Real^J)$ which is an extension operator from surface functions to volumetric functions so that $\E_S[f]\vert_S \approx f$ and $\E_S[f]\too 0$ away from $S$:  
\begin{equation}\label{e:E_S_def}
\E_S[f](x) = \frac{1}{2\pi \sigma^2}\int_{S} f(y) \Phi_\sigma(|x-y|) \,da(y),
\end{equation}
where $da$ is the area element of the surface $S$.

The operator $\E_S$  enjoys several favorable approximation properties: First, \begin{equation}\label{e:E_S_convergence}
		\E_S[f](x) \xrightarrow{\sigma\too 0}\begin{cases}
			f(x) & x\in S \\ 0 & \mathrm{otherwise}
		\end{cases}.		
	\end{equation}
	That is, $\E_S[f]$ approximates $f$ over $S$ and decays to zero away from $S$. In particular, this implies that the constant one function, $\ones:S\too \Real$, satisfies
	\begin{equation}\label{e:E_S_convergence_indicator}
		\E_S[\ones](x) \too \chi_S(x),
	\end{equation}
    where $\chi_S(x)$ is the volumetric indicator function of $S\subset \Real^3$. Interestingly, $\E_S[\ones]$ provides also higher-order geometric information of the surface $S$,
	\begin{equation}\label{e:E_S_mean_curvature}
		\nabla \E_S[\ones]\Big\vert_S \too -H\cdot n,
	\end{equation}
	where $H:S\too \Real$ is the mean curvature function of $S$ and $n:S\too S^2$ ($S^2\subset \Real^3$ is the unit sphere) is the normal field to $S$.

We prove that the approximation quality in \eqref{e:EX_approx_ES} improves as the point cloud sample $X\subset S$ densifies $S$, and the operator $\E_X$ becomes more and more consistent. In that case $\E_X[\ones]$ furnishes an approximation to the indicator function of the surface $S$ and its gradient, $\nabla \E_X[\ones]$, to the mean curvature vectors of $S$. This demonstrates that given the simplest, all ones input data $\ones\in\Real^{I\times 1}$, the network can already reveal the indicator function and the mean curvature vectors of the underlying surface by simple linear operators corresponding to specific choices of the kernel $k$ in \eqref{e:kernel}. 

These results are summarized in the following theorem which is proved in appendix \ref{app:proofs}.
\begin{theorem} \label{thm:extension_op_properties}
	Let $f\in C(S,\Real^J)$ be a continuous function defined on a compact smooth surface $S\subset \Real^3$. The extension operator \eqref{e:extension},\eqref{e:ell} with
	\begin{equation}\label{e:choice_of_c}
		c = \frac{1}{2\pi\sigma^2},
	\end{equation}
	and
	\begin{subequations}
		\begin{align} \label{e:omega}
			\omega_i &= \area (\Omega_i), \\ \label{e:Omega}
			\Omega_i&=\set{y\in S \ \vert \ d_S(y-x_i)\leq d_S(y-x_{i'}),\  \forall i'},
		\end{align}
	\end{subequations}
	the Voronoi cell of $x_i\in S$, where $d_S$ denotes the distance function of points on $S$, satisfies
	\begin{equation}\label{e:E_approx_E_S}
		\E_X\circ \R_X[f](x) \too  \E_S[f](x),
	\end{equation}
	where $X\subset S$ is a $\delta$-net and $\delta\too 0$. Furthermore, $\E_S$ satisfies the approximation and mean curvature properties as defined in \eqref{e:E_S_convergence}, \eqref{e:E_S_convergence_indicator}, \eqref{e:E_S_mean_curvature}. 	
	\end{theorem}

\subsection{Revisiting image CNNs} Our model is a generalization of image CNNs. Images can be viewed as point clouds in regular grid configuration, $X=\set{x_i}\subset \Real^3$, with image intensities $\mathcal{I}_i$ as functions over this point cloud,
$$\E_X(\mathcal{I})=\sum_{i} \mathcal{I}_i \Phi(x-x_i),$$
where $\Phi$ is the indicator function over one square grid cell (\ie, pixel). 
In this case the extension operator reproduces the image as a volumetric function over $\Real^2$. Writing the convolution kernel also in the basis $\Phi$ with regular grid translations leads to \eqref{e:key} reproducing the standard image discrete convolution.

\section{Experiments}

We have tested our PCNN framework on the problems of point cloud classification, point cloud segmentation, and point cloud normal estimation. We also evaluated the different design choices and network variations.

\begin{figure}
	\includegraphics[width=0.9\columnwidth]{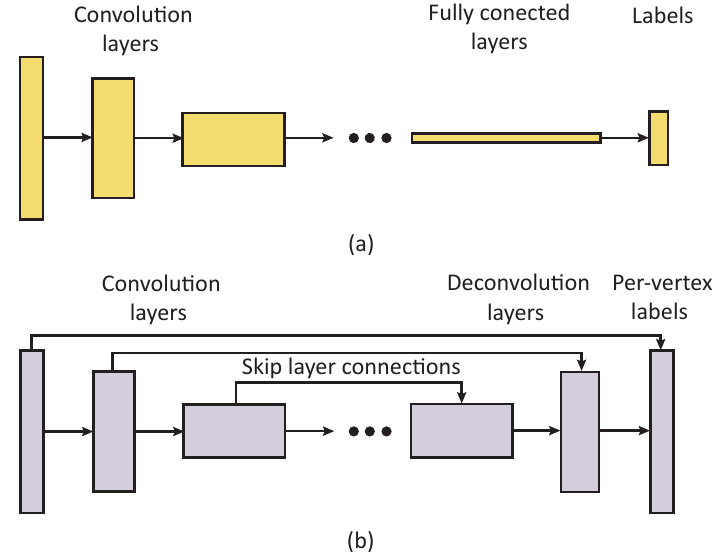}
	\caption{Different architectures used in the paper: (a) classification network; and (b) segmentation network.}\label{fig:arch}
\end{figure}

\subsection{Point cloud classification}
We tested our method on the standard ModelNet40 and ModelNet10 benchmarks \cite{wu20153d}. ModelNet40 is composed of 9843 train models and 2468 test models in 40 different classes, such as guitar, cone, laptop etc. ModelNet 10 consists 3991 train and 908 test models from ten different classes. The models are originally given as triangular meshes. The upper part of Table \ref{tab:classifictionRes} compares our classification results versus state of the art learning algorithms that use only the point clouds (\ie, coordinates of points in $\Real^3$) as input: PointNet \cite{qi2016pointnet}, PointNet++ \cite{qi2017pointnet++}, deep sets \cite{zaheer2017deep}, ECC\cite{simonovsky2017dynamic} and kd-network\cite{klokov2017escape}. For completeness we also provide results of state of the art algorithms taking as input additional data such as meshes and normals. Our method outperforms all point cloud methods and all other non-ensemble methods.

We use the point cloud data of \cite{qi2016pointnet,qi2017pointnet++} that sampled a point cloud from each model using farthest point sampling. In the training we randomly picked 1024 farthest point sample out of a fixed set of 1200 farthest point sample for each model. As in \cite{klokov2017escape} we also augment the data with random anisotropic scaling in the range $[-0.66,1.5]$ and uniform translations in the range $[-0.2,0.2]$. As input to the network we provide the constant one tensor, together with the coordinate functions of the points, namely $(1,x)\in \Real^{I\times 4}$. The $\sigma$ parameter controls the variance of the RBFs (both in the convolution kernels and the extension operator) and is chosen to be $\sigma=I^{-1/2}$. The translations of the convolution are chosen to be regular $3\times 3 \times 3$ grid with size $2\sigma$.

At test time, similarly to \cite{klokov2017escape} we use voting: we sample ten different samples of size 1024 from  1200 points on each point cloud, apply anisotropic scaling, propagate it through the net and sum the label probability vectors before taking the label with the maximal probability.

We used standard convolution architecture, see Figure \ref{fig:arch}:
\begin{align*}
&\mathrm{conv\_block}(1024,256,64) \rightarrow {\mathrm{conv\_block}(256,64,256)} \\ & \rightarrow  \mathrm{conv\_block}(64,1,1024) \rightarrow \mathrm{fully\_connected\_block},
\end{align*}
where conv\_block(\#points in, \#points out ,\#channels) consists of a convolution layer, batch normalization, Relu activation and pooling. The fully connected block is a concatenation of a two fully connected layers with dropout after each one.

\begin{table*}[h]
	\centering
	\begin{tabular}{lccc}
		\toprule
		algorithm & \# points & 10 models & \multicolumn{1}{l}{40 models} \\
		\midrule		
		\textbf{Point cloud methods} &&&\\
		
		\midrule
		pointnet \cite{qi2016pointnet}& \multicolumn{1}{c}{1024} & -    & 89.2 \\
		pointnet++ \cite{qi2017pointnet++}& \multicolumn{1}{c}{1024} & -    & 90.7 \\
		deep sets \cite{zaheer2017deep}& \multicolumn{1}{c}{1000} & -    & 87.1 \\
		ECC  \cite{simonovsky2017dynamic} & \multicolumn{1}{c}{1000} & \multicolumn{1}{c}{90.8} & 87.4 \\
		kd-network \cite{klokov2017escape} & \multicolumn{1}{c}{1024} & \multicolumn{1}{c}{93.3} & 90.6 \\
		kd-network \cite{klokov2017escape} & \multicolumn{1}{c}{32k} & \multicolumn{1}{c}{94.0} & 91.8 \\
		ours  & \multicolumn{1}{c}{1024} & \multicolumn{1}{c}{\textbf{94.9}} & \textbf{92.3}  \\
		\midrule
		\textbf{Additional input features} &&&\\
		\midrule		
		FusionNet (uses mesh structure) \cite{hegde2016fusionnet} & -    & 93.1    & 90.8 \\
		VRN single(uses mesh structure) \cite{brock2016generative}& -    & -    & 92.0 \\
		OctNet \cite{riegler2016octnet} & -    & \multicolumn{1}{c}{90.9} & 86.5 \\
		
		VRN ensemble(uses mesh structure) \cite{brock2016generative} & -    & \multicolumn{1}{c}{97.1} & 95.5 \\
		MVCNN (uses mesh structure)\cite{qi2016volumetric} & -    & -    & 92.0 \\
		MVCNN-MultiRes(uses mesh structure)\cite{qi2016volumetric} & -    & -    & 93.8 \\
		pointnet++ (uses normals) \cite{qi2017pointnet++}& 5K    & -    & 91.9 \\
		
		OCNN (uses normals) \cite{wang2017cnn}  &- &- & 90.6\\
		\bottomrule
	\end{tabular}	\caption{Shape classification results on the ModelNet40 and ModelNet10 datasets.} \label{tab:classifictionRes}	
\end{table*}

\paragraph{Robustness to sampling.} The inset compares our method with \cite{qi2016pointnet,qi2017pointnet++} when feeding a trained $1024$ point model on sparser test point clouds of size $k=1024,512,256,128,64$.
\begin{wraptable}[8]{r}{0.45\columnwidth}
	\vspace{-0pt}\hspace{-0pt}
	\includegraphics[width=0.45\columnwidth]{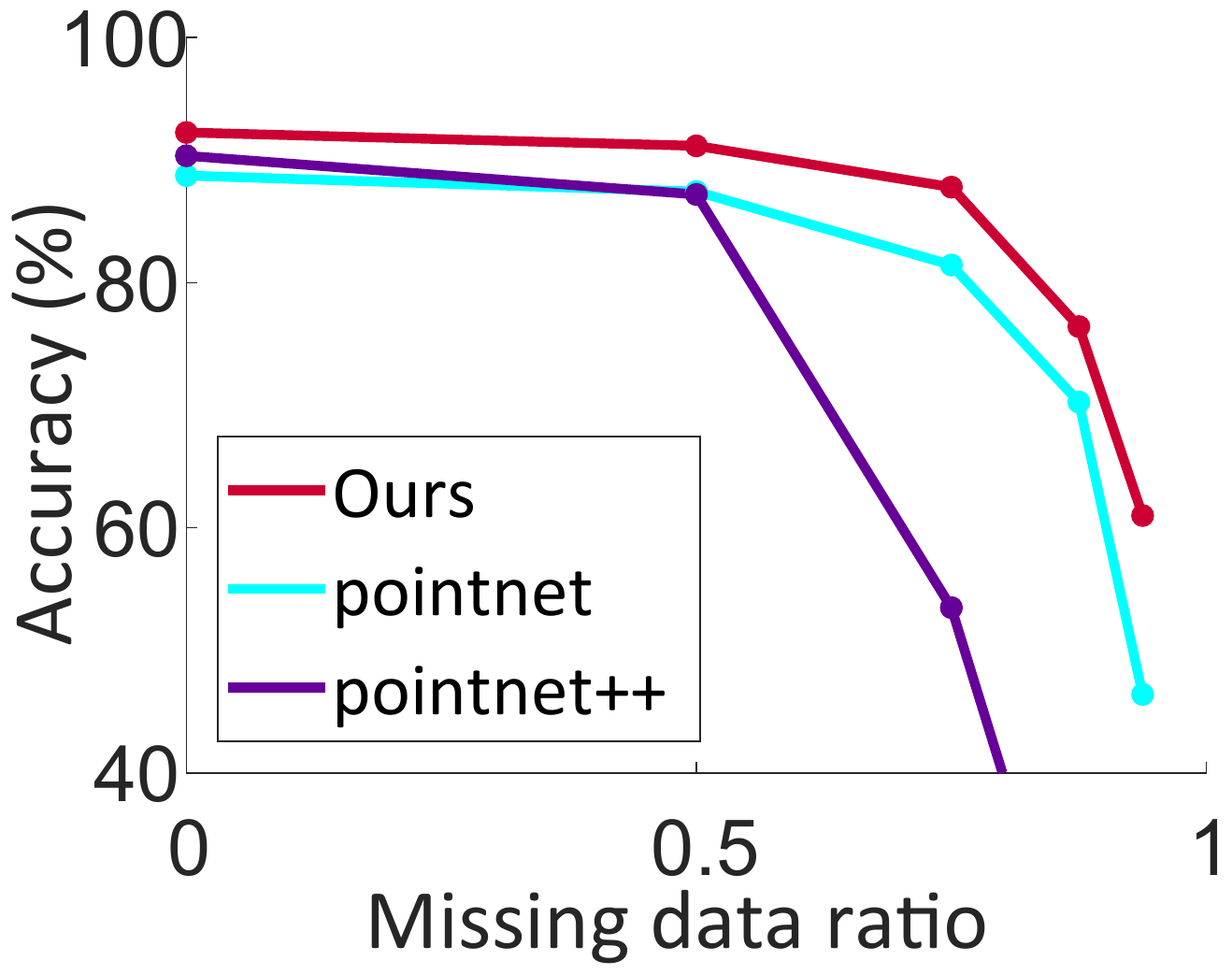}	
	\vspace{-0cm}
\end{wraptable}
The favorable robustness of our method to sub-sampling can be possibly explained by the fact that our extension operator possess approximation power, even with sparse samples, \eg, for smooth shapes, see Figure \ref{fig:extension_op}.

\begin{table}[h]
	\centering
	\begin{tabular}{lrr}
		\toprule
		\multicolumn{3}{c}{Method variations on ModelNet40} \\
		\midrule
		Variation & \multicolumn{1}{l}{\# points} & \multicolumn{1}{l}{accuracy} \\
		\midrule
		Less points  & 256   &  90.8\\		
		Less points  & 512   &  91.2 \\
		Less points & 870 & 92.2 \\		
		More points  & 1800  &  92.3 \\		
		Interpolation & 1024 & 92.0\\
		Spherical kernel & 1024  & 91.5 \\
		Learned translations & 1024  & 91.3 \\		
		Indicator input & 1024 & 91.2 \\
		xyz only input & 1024 & 91.5 \\		
		Less parameters &  1024 & 92.2 \\
		Small sigma & 1024 & 85.5 \\
		\bottomrule
	\end{tabular}%
	\caption{Classification with variations to the PCNN model.}
	\label{tab:classifictionResOurVariation}%
\end{table}%

\paragraph{Method variants.} We evaluate the performance of our algorithm subject to the variation in: the number of points $I$, the kernel translations $\set{y_l}$, the input tensor $f$, different bases $\ell_i$ in \eqref{e:extension}, the choice of $\sigma$, and number of learnable parameters. Points were randomly sampled by the same ratio as in the above (\eg~512 out of 600).

Table \ref{tab:classifictionResOurVariation} presents the results.
Using the constant one as input $f=\ones\in \Real^{I\times 1}$ provides almost comparable results to using the full coordinates of the points $f=(\ones,x)\in\Real^{I\times 4}$. This observation is partially supported by the theoretical analysis shown in section \ref{s:preperties} which states that our extension operator applied to the constant one tensor already provides good approximation to the underlying surface and its normal as well as curvature.
Using interpolation basis $\set{\ell_i}$ in the extension operator \eqref{e:extension}, although heavier computationally, does not provide better results.
Applying too small $\sigma$ provides worse classification result. This can be explained by the observation that small $\sigma$ results in separated Gaussians centered at the points which deteriorates the approximation ($X$ and $\sigma$ should be related).
Interestingly, using a relatively small network size of 1.4M parameters provides comparable classification result.

\begin{figure}[t]
	\includegraphics[width=\columnwidth]{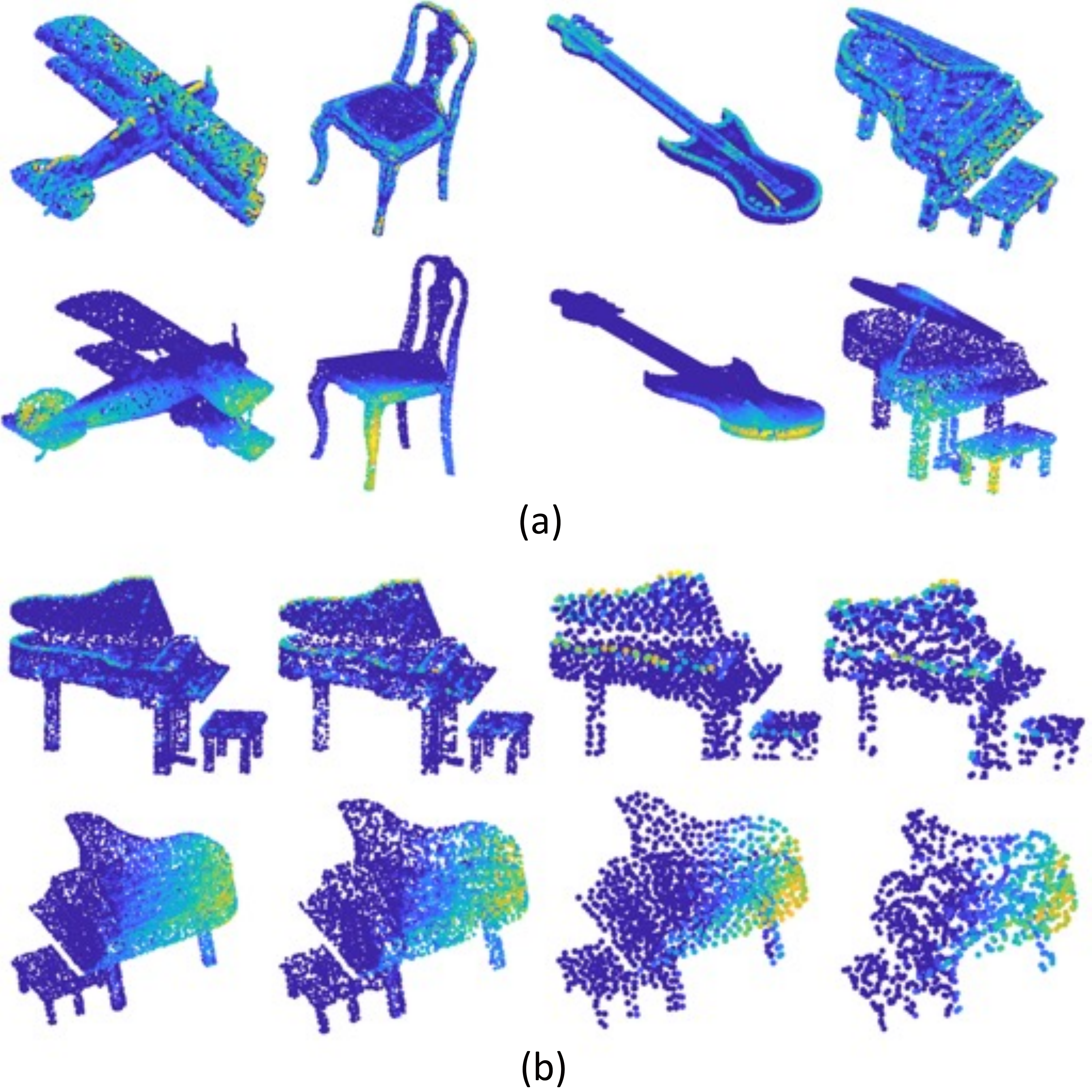}
	\caption {Our point cloud convolution is translation invariant and robust to sample size and density: (a) shows feature activations of two kernels (rows) learned by our network's first convolution layer on different shapes (columns). The features seems consistent across the different models; (b) shows another pair of kernels (rows) on a single model with varying sampling density (from left to right): $10K$ points, $5K$ points (random sampling), 1K points (farthest point sampling) and $1K$ (random sampling). Note that the convolution captures the same geometric properties on all models regardless of the sampling.}\label{fig:feat1}
\end{figure}

\paragraph{Feature visualizations.} Figure \ref{fig:feat1} visualizes the features learned in the first layer of PCNN on a few shapes from the ModelNet40 dataset. As in the case of images, the features learned on the first layer are mostly edge detectors and directional derivatives. Note that the features are consistent through different sampling and shapes. Figure \ref{fig:highlevelfeats} shows 9 different features learned in the third layer of PCNN. In this layer the features capture more semantically meaningful parts.

\begin{figure}
	\includegraphics[width=\columnwidth]{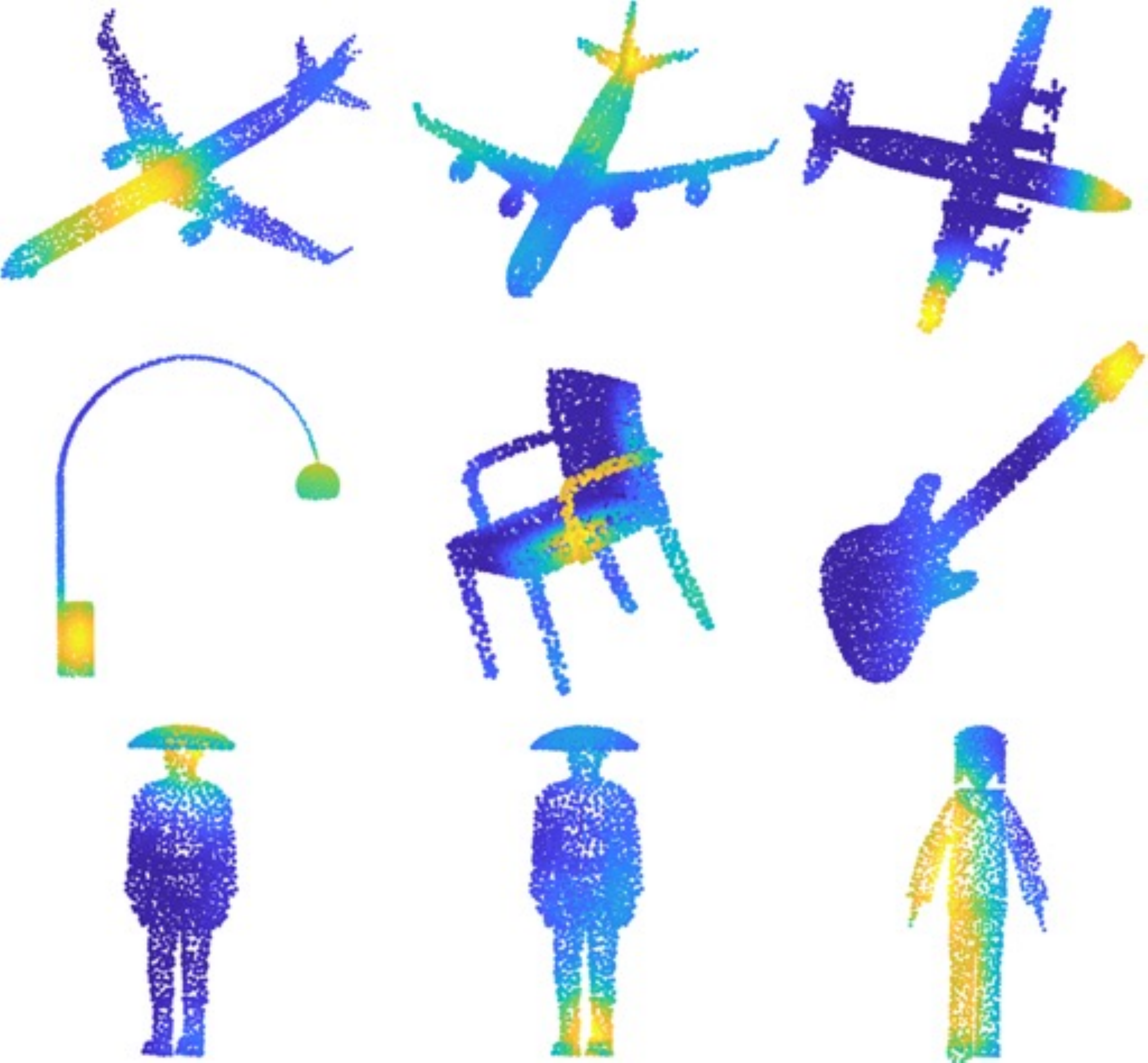}
	\caption {High level features learned by PCNN's third convolution layer and visualized on the input point cloud. As expected, the features are less geometrical than the first layer's features (see Figure \ref{fig:feat1}) and seem to capture more semantically meaningful shape parts.}\label{fig:highlevelfeats}
\end{figure}

\subsection{Point cloud segmentation}\label{s:segmentation}
Our method can also be used for part segmentation: given a point cloud that represents a shape the task is to label each point with a correct part label. We evaluate PCNN performance on ShapeNet part dataset \cite{yi2016scalable}. ShapeNet contains 16,881 shapes from 16 different categories, and total of 50 part labels.

Table \ref{tab:segres} compares per-category and mean IoU(\%) scores of PCNN with state of the art point cloud methods: PointNet \cite{qi2016pointnet}, kd-network \cite{klokov2017escape}, and 3DCNN (results taken from \cite{qi2016pointnet}). Our method outperforms all of these methods. For completeness we also provide results of other methods that use additional shape features or mesh normals as input. Figure \ref{fig:seg} depicts several of our segmentation results.

For this task we used standard convolution segmentation architecture, see Figure \ref{fig:arch}:
\begin{align*}
&\mathrm{conv\_block}(2048,512,64) \rightarrow {\mathrm{conv\_block}(512,128,128)} \\
&\rightarrow \mathrm{conv\_block}(128,16,256) \rightarrow {\mathrm{deconv\_block}(16,128,512)} \\ & \rightarrow  \mathrm{deconv\_block}(128,512,256) \rightarrow  \mathrm{deconv\_block}(512,2048,256)\\ & \rightarrow  \mathrm{deconv\_block}(2048,2048,256) \rightarrow \mathrm{conv\_block}(2048,2048,50),
\end{align*}
where deconv\_block(\#points in,\#points out,\#features) consists of an upsampling layer followed by a convolution block.
In order to provide the last layers with raw features we also add skip-layers connections, see Figure \ref{fig:arch}(b). This is a common practice in such architectures where fine details are needed at the output layer (\eg, \cite{cciccek20163d}).

We use the data from \cite{qi2016pointnet} (2048 uniformly sampled points on each model). As done in \cite{qi2016pointnet} we use a single network to predict segmentations for each of the object classes and concatenate a hot-one encoding of the object's label to the bottleneck feature layer. At test time, we use only the part labels that correspond to the input shape (as in \cite{qi2016pointnet, klokov2017escape}).

\begin{table*}[t]
	\vspace{-1.8cm}
	\resizebox{\textwidth}{!}{%
	\centering {\footnotesize	\begin{tabular}{@{\hskip0pt}l@{\hskip2pt}l@{\hskip2pt}r@{\hskip2pt}r@{\hskip2pt}r@{\hskip2pt}r@{\hskip2pt}r@{\hskip2pt}r@{\hskip2pt}r@{\hskip2pt}r@{\hskip2pt}r@{\hskip2pt}r@{\hskip2pt}r@{\hskip2pt}r@{\hskip2pt}r@{\hskip2pt}r@{\hskip2pt}r@{\hskip2pt}r@{\hskip2pt}r@{\hskip0pt}}
			\toprule
			& \multicolumn{1}{l}{input} & \multicolumn{1}{l}{mean} & \multicolumn{1}{l}{aero} & \multicolumn{1}{l}{bag} & \multicolumn{1}{l}{cap} & \multicolumn{1}{l}{car} & \multicolumn{1}{l}{chair} & \multicolumn{1}{l}{ear-p} & \multicolumn{1}{l}{guitar} & \multicolumn{1}{l}{knife} & \multicolumn{1}{l}{lamp } & \multicolumn{1}{l}{laptop} & \multicolumn{1}{l}{motor} & \multicolumn{1}{l}{mug} & \multicolumn{1}{l}{pistol} & \multicolumn{1}{l}{rocket} & \multicolumn{1}{l}{skate} & \multicolumn{1}{l}{table} \\
			\midrule
			\textbf{Point clouds}\\
			\midrule
			\textbf{Ours}&   2K pnts  & \textbf{85.1}  &    82.4   &  \textbf{80.1}     &  \textbf{85.5}     &  \textbf{79.5}     &  \textbf{90.8}     &  73.2     &    91.3   & 86.0     &  \textbf{85.0}     &  \textbf{95.7}     &  \textbf{73.2}     &  \textbf{94.8}     &   \textbf{83.3}    &  51.0     &  \textbf{75.0}     & \textbf{81.8} \\
			
			PointNet &   2K pnts & 83.7  & \textbf{83.4}  & 78.7  & 82.5  & 74.9  & 89.6  & 73    & \textbf{91.5}  & 85.9  & 80.8  & 95.3  & 65.2  & 93    & 81.2  & 57.9  & 72.8  & 80.6 \\
			kd-network&   4K pnts & 82.3  & 80.1  & 74.6  & 74.3  & 70.3  & 88.6  & \textbf{73.5}  & 90.2  & \textbf{87.2}  & 81    & 94.9  & 57.4  & 86.7  & 78.1  & 51.8  & 69.9  & 80.3 \\
			3DCNN &    & 79.4  & 75.1  & 72.8  & 73.3  & 70    & 87.2  & 63.5  & 88.4  & 79.6  & 74.4  & 93.9  & 58.7  & 91.8  & 76.4  & 51.2  & 65.3  & 77.1 \\
			\midrule
			\textbf{Additional input}\\
			\midrule
			SyncSpecCNN&   sf   & 84.7  &       &       &       &       &       &       &       &       &       &       &       &       &       &       &       &  \\
			Yi &   sf    & 81.4  & 81    & 78.4  & 77.7  & 75.7  & 87.6  & 61.9  & 92.0    & 85.4  & 82.5  & 95.7  & 70.6  & 91.9  & 85.9  & 53.1  & 69.8  & 75.3 \\
			PointNet++   &   pnts, nors   & 85.1  & 82.4    & 79.0  & 87.7  & 77.3  & 90.8  & 71.8  & 91.0    & 85.9  & 83.7  & 95.3  & 71.6  & 94.1  & 81.3  & 58.7  & 76.4  & 82.6 \\
			OCNN (+CRF)   &   nors  & 85.9  & 85.5    & 87.1  & 84.7  & 77.0  & 91.1  & 85.1  & 91.9    & 87.4  & 83.3  & 95.4  & 56.9  & 96.2  & 81.6  & 53.5  & 74.1  & 84.4 \\		
			\bottomrule
	\end{tabular}}}
	\caption{ShapeNet segmentation results by point cloud methods (top) and methods using additional input data (sf - shape features; nors - normals). The methods compared to are: PointNet  \protect\cite{qi2016pointnet}; kd-network  \cite{klokov2017escape}; 3DCNN \cite{qi2016pointnet}; SyncSpecCNN  \cite{yi2016syncspeccnn}; Yi \cite{yi2016scalable};	PointNet++ \cite{qi2017pointnet++}; OCNN (+CRF refinement) \cite{wang2017cnn}.}\label{tab:segres}
\end{table*}%

\begin{figure}[t]
	\vspace{-0.29cm}
	\includegraphics[width=\columnwidth]{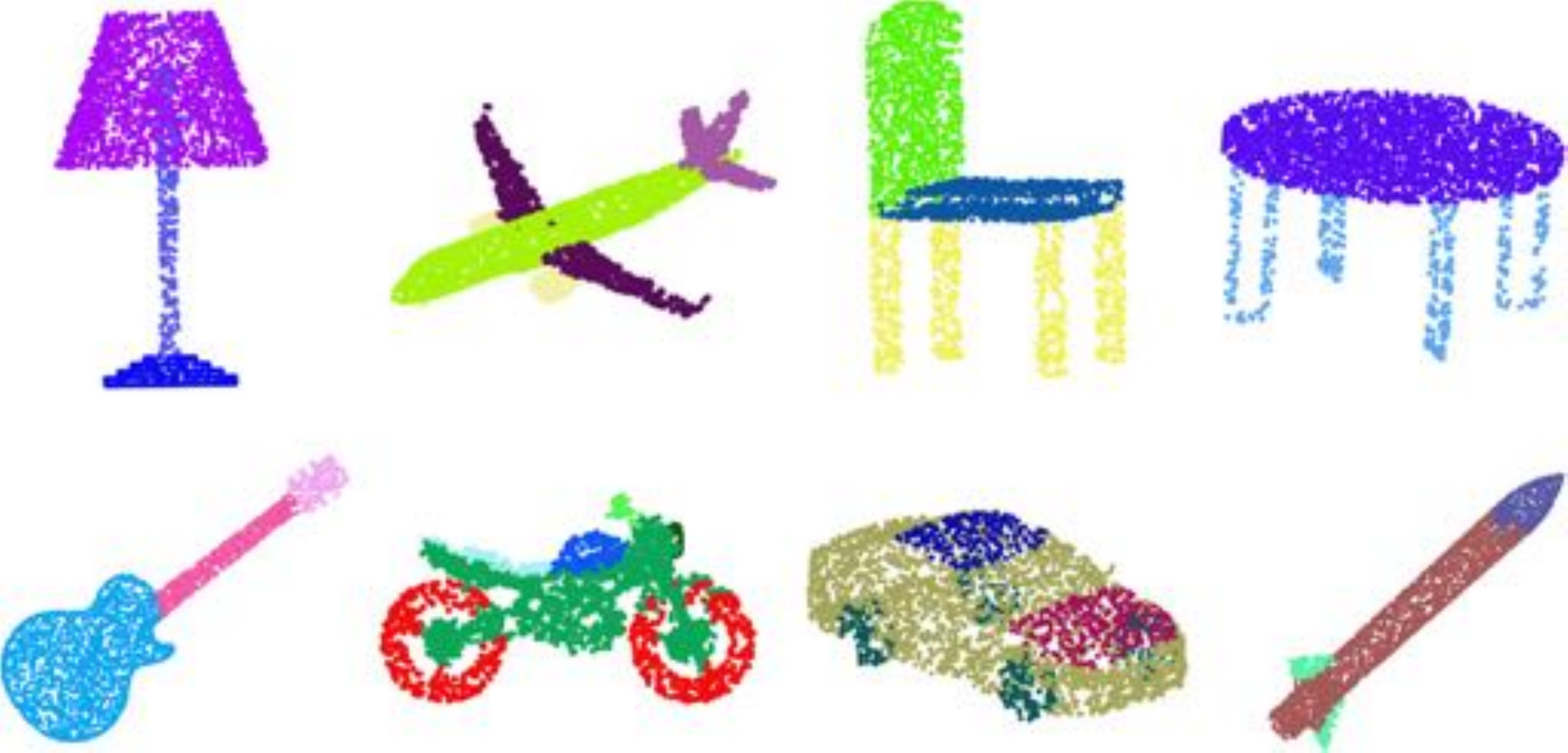}
	\caption {Results of PCNN on the part segmentation benchmark from \cite{yi2016scalable} }\label{fig:seg}
\end{figure}

\subsection{Normal estimation}
Estimating normals of a point cloud is a central sub-problem of the 3D reconstruction problem. We cast this problem as supervised regression problem and employ segmentation network with the following changes: the output layer is composed of $3$ channels instead of $50$ which are then normalized and fed into cosine-loss with the ground truth normals.

We have trained and tested our network on the standard train/test splits of the ModelNet40 dataset (we used the data generator code by \cite{qi2017pointnet++}). Table \ref{tab:normalRes} compares the mean cosine loss (distance) of PCNN and the normal estimation of \cite{qi2016pointnet} and \cite{qi2017pointnet++}. Figure \ref{fig:normals} depicts normal estimation examples from this challenge.

\begin{table}[htbp]
	\centering
	\begin{tabular}{rlrl}
		\toprule
		\multicolumn{1}{l}{Data } & algorithm & \multicolumn{1}{l}{\# points} & \multicolumn{1}{l}{error} \\
		\midrule
		\multicolumn{1}{l}{ModelNet40} & PointNet & 1024  & 0.47 \\
		& PointNet++ & 1024  & 0.29 \\
		& \textbf{ours}  & \textbf{1024}  &  \textbf{0.19}\\
		\bottomrule
	\end{tabular}%
	\caption{Normal estimation in ModelNet40.}
	\label{tab:normalRes}%
\end{table}%

\begin{figure}
	\vspace{-0.85cm}
	\includegraphics[width=\columnwidth]{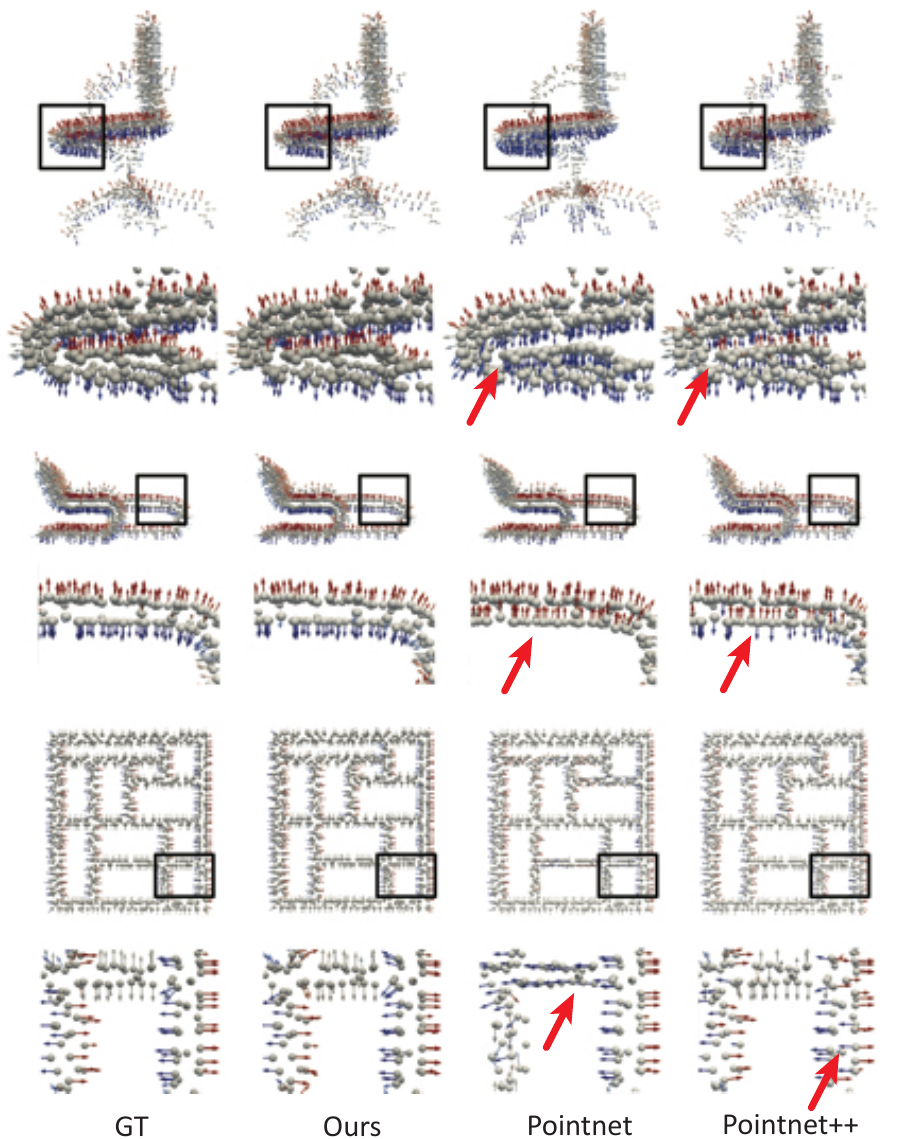}
	\caption {Normal estimation in ModelNet40. We show normal estimation of four models (rows) with blow-ups. Normals are colored by one of their coordinates for better visualization. Note that competing methods sometimes fail to recognize the outward normal direction (examples indicated by red arrows).}\label{fig:normals}
\end{figure}


\subsection{Training details, timings and network size}

We implemented our method using the TensorFlow library \cite{tensorflow2015-whitepaper} in Python. We used the Adam optimization method with learning rate 0.001 and decay rate 0.7. The models were trained on Nvidia p100 GPUs. Table \ref{tab:time} summarizes running times and network sizes.
%
%
Our smaller classification network achieves state of the art result (see Table \ref{tab:classifictionResOurVariation}, previous to last row) and has only 1.4M parameters with a total model size of 17 MB.


\begin{table}
	\centering
	\begin{tabular}{@{\hspace{1.5pt}}l@{\hspace{1.5pt}}r@{\hspace{1.5pt}}r@{\hspace{1.5pt}}r@{\hspace{1.5pt}}r@{\hspace{1.5pt}}r@{\hspace{1.5pt}}}
		\toprule
		& \#Param.  & \begin{tabular}{@{\hspace{1.5pt}}r@{\hspace{1.5pt}}}
			Size \\ (mb)
		\end{tabular}  &
		\begin{tabular}{@{\hspace{1.5pt}}r@{\hspace{1.5pt}}}
			Converge \\ (epochs)
		\end{tabular} &
		\begin{tabular}{@{\hspace{1.5pt}}r@{\hspace{1.5pt}}}
			Training \\ (min)
		\end{tabular} &
		\begin{tabular}{@{\hspace{1.5pt}}r@{\hspace{1.5pt}}}
			Forward \\ (msec)
		\end{tabular} \\
		\midrule
		\textbf{Classification}\\
		\midrule
		Large & 8.1M & 98 & 250 &  6 & 80 \\
		Small & 1.4M & 17 & 250 &  5 & 70 \\
		\midrule
		\textbf{Segmentation}\\
		\midrule
		& 5.4M  & 62  &  85 & 37 & 200 \\		
		\bottomrule
	\end{tabular}%
	\caption{Timing and network size. Training time is measured in minuets per epoch.}
	\label{tab:time}%
\end{table}%

\section{Conclusions}
This paper describes PCNN: a methodology for defining convolution of functions over point clouds that is efficient, invariant to point cloud order, robust to point sampling and density, and posses translation invariance. The key idea is to translate volumetric convolution to arbitrary point clouds using extension and restriction operators.

Testing PCNN on standard point cloud benchmarks show state of the art results using compact networks. The main limitation of our framework compared to image CNN is the extra computational burden due to the computation of farthest point samples in the network and the need to compute the ``translating'' tensor $q\in \Real^{I\times I \times L}$ which is a function of the point cloud $X$. Still, we believe that a sparse efficient implementation can alleviate this limitation and mark it as a future work. Other venues for future work is to learn kernel translations per channel similarly to \cite{dai2017deformable}, and apply the method to data of higher dimension than $d=3$ which seems to be readily possible. Lastly, we would like to test this framework on different problems and architectures.

\section{Acknowledgements}
This research was supported in part by the European Research Council (ERC Consolidator Grant, "LiftMatch" 771136), the Israel Science Foundation (Grant No. 1830/17). We would like thank the authors of PointNet \cite{qi2016pointnet} for sharing their code and data.

\bibliographystyle{abbrv}
\bibliography{paper1}

\appendix


\section{Proofs}\label{app:proofs}

\subsection{Multiplication law for Gaussians}
\begin{proposition}
	Let
	$$\Phi_{\mu,\sigma}(x)=\exp(-\frac{\|x-\mu\|^2}{2\sigma^2})$$
	and
	$$B(\sigma)=\frac{1}{(2\pi \sigma^2)^\frac{3}{2}}$$
	then
	$$ \Phi_{\mu_1,\sigma_1}*\Phi_{\mu_2,\sigma_2} = C(\sigma_1,\sigma_2)\cdot\Phi_{\mu,\sigma}$$
	where $\mu=\mu_1+\mu_2$, $\sigma=\sqrt{\sigma_1^2+\sigma_2^2}$ and $C(\sigma_1,\sigma_2)=\frac{B(\sigma_1)B(\sigma_2)}{B(\sigma)}$
\end{proposition}

\begin{proof}
	It is well known \cite{weisstein2000normal} that the convolution of two normal distributions is again a normal distribution:
	$$ B(\sigma_1)\Phi_{\mu_1,\sigma_1}*B(\sigma_2)\Phi_{\mu_2,\sigma_2} = B(\sqrt{\sigma_1^2+\sigma_2^2})\cdot\Phi_{\mu,\sigma}$$
	The result above follows from the linearity of the convolution.
\end{proof}

\subsection{Theoretical properties of the extension operator}

\begin{proof}[Proof of theorem \ref{thm:extension_op_properties}]
	Let us first show \eqref{e:E_approx_E_S}. Denote $g_x(y)=f(y)\Phi_\sigma(|x-y|)$. By Lemma \ref{lem:equicontinuous} for arbitrary $\epsilon>0$ there exists $\delta>0$ so that $d_S(y,x)<\delta$ implies $\abs{g_x(y)-g_x(x)}<\epsilon$ for all $x\in \Real^3$. Taking $X$ to be $\delta$-net of $S$ we get that
	$$\Big | \E_S[f](x) - \E_X\circ \R_X[f](x) \Big |\leq \epsilon \frac{\sum_i\area(\Omega_i)}{2\pi \sigma^2}  \leq \epsilon \frac{ \area(S)}{2\pi \sigma^2}.$$
	
	To show \eqref{e:E_S_convergence} first let $x\notin S$. Then as $\sigma\too 0$ we have $\max_{y\in S}\Phi_\sigma(|y-x|)\too 0$ and therefore $\E_S[f](x) \too 0$. Next consider $x\in S$. It is enough to show that
	$$\frac{1}{2\pi \sigma^2}\int_S \Phi_\sigma(|x-y|)da(y) \xrightarrow{\sigma\too 0} 1.$$
	Indeed, let $\epsilon>0$. Since $f$ is uniformly continuous, take $\delta>0$ sufficiently small so that $\abs{f(x)-f(x')}<\epsilon$ if $d_S(x,x')<\delta$.  Take $\sigma>0$ sufficiently small so that $$\frac{1}{2\pi \sigma^2}\int_{S \setminus B(x,\delta) }\Phi_\sigma(|x-y|)da(y) \leq \epsilon,$$
	where $B(x,\delta) = \set{y \ \vert \ |y-x|<\delta}$ and
	$$\Big | \frac{1}{2\pi \sigma^2}\int_{S}\Phi_\sigma(|x-y|)da(y) - 1 \Big | \leq \epsilon.$$
	Hence
	$$\Big | \frac{1}{2\pi \sigma^2}\int_{S \cap B(x,\delta) }\Phi_\sigma(|x-y|)da(y) -1 \Big | \leq 2\epsilon.$$
	Therefore,
	\begin{align*}
	\abs{\frac{1}{2\pi \sigma^2}\int_{S}f(y)\Phi_\sigma(|x-y|)da(y) - f(x)} &\leq 3\epsilon (1+|f|_{\infty}).
	\end{align*}	
	Lastly, to show \eqref{e:E_S_mean_curvature} note that
	$$\nabla_x \E_S[\ones](x)= -\frac{1}{2\pi \sigma^4} \int_S (x-y) \Phi(|x-y|) da(y).$$ Using an argument from \cite{belkin2005towards} (see Section 4.2) where we take $\sigma^2=2t$ in their notation and get convergence to $-\frac{1}{2}\Delta_S x$, where $\Delta_S$ is the Laplace-Beltrami operator on surfaces $S$. To finish the proof remember that  \cite{burago2013geometry}
	$$-\frac{1}{2}\Delta_S x = - H\cdot n.$$	
\end{proof}

\begin{lemma}\label{e:delta_approximation}
	Let $S\subset \Real^3$ be a compact smooth surface. Then,
	\begin{equation}
	\lim_{\sigma\too 0} \frac{1}{2\pi \sigma^2}\int_S \Phi_\sigma(|x-y|)da(y) = 1
	\end{equation}
\end{lemma}
\begin{proof}
	Denote $T_x S$ the tangent plane to $S$ centered at $x\in S$. Let $y=y(u):T_x S \too S$ be the local parameterization to $S$ over $T_x S$, where $u$ is the local coordinate at $T_x S$.  Since $S$ is smooth and compact we have that $\forall u \in \Upsilon_\delta=T_xS \cap B(x,\delta) $, 
	\begin{align} \label{e:delta_approximation_yu_minus_u}
	& \abs{y(u)-u}=\OO(\delta^2)\\ \label{e:delta_approximation_det_dy_u_minus_1}
	& \abs{\abs{dy(u)}-1} = \OO(\delta),
	\end{align}
	where $|dy(u)|$ is the pulled-back area element of $S$ \cite{do2017differential}.
	We break the error to $\abs{\frac{1}{2\pi\sigma^2}\int_S \Phi_\sigma(|x-y|)da(y) -1  }$
	\begin{align*}
	&\leq \underset{(i)}{\frac{1}{2\pi\sigma^2} \int_{S\setminus y(\Upsilon_\delta)}\Phi_\sigma da} + \underset{(ii)}{\frac{1}{2\pi\sigma^2}\abs{\int_{y(\Upsilon_\delta)} \Phi_\sigma da - \int_{\Upsilon_\delta} \Phi_\sigma du }} \\
	& \underset{(iii)}{\abs{\frac{1}{2\pi \sigma^2} \int_{T_x S} \Phi_\sigma du  - 1}} +
	\underset{(iv)}{\frac{1}{2\pi\sigma^2}\int_{T_x S\setminus \Upsilon_\delta} \Phi_\sigma du}.
	\end{align*}
	First, we note that $(iii)=0$. Now, take $\delta = \sigma^{1-\tau}$, for some fixed $0<\tau<1$, where $\sigma>0$. Then, $(i)=\OO(\sigma)$, $(iv)=\OO(\sigma)$. Lastly $(ii)$
	\begin{align*}
	&\leq \frac{1}{2\pi \sigma^2} \int_{\Upsilon_\delta} \Big |\Phi_\sigma(|y(u)|)-\Phi_\sigma(|u|)\Big ||dy(u)|du \\&  + \frac{1}{2\pi \sigma^2}\int_{\Upsilon_\delta} \Phi_\sigma(|u|)\Big | \abs{dy(u)} -1 \Big |du \\
	& \leq \frac{1}{2\pi \sigma^2} \frac{\max_{u\in \Upsilon_\delta}\Big | |y(u)|-|u| \Big | }{\sigma e^{1/2}}(1+\OO(\delta))\OO(\delta^2) + \OO(\delta) \\
	& = \OO(\sigma^{1-4\tau}),
	\end{align*}
	where we used Lemma \ref{lem:gaussian} in the last inequality. Taking any $\tau<1/4$ proves the result.
\end{proof}

\begin{lemma}\label{lem:equicontinuous}
	Let $f\in C(S,\Real)$, with $S\subset \Real^3$ a compact surface. The family of functions $\set{g_x}_{x\in\Real^3}$ defined by $g_x(y)=f(y)\Phi_\sigma(|x-y|)$, $y\in S$ is uniformly equicontinuous.
\end{lemma}
\begin{proof} $\abs{g_x(y)-g_x(y')}$
	\begin{align*}
	& \leq \abs{f(y)-f(y')}\Phi_\sigma(|x-y|) + \\
    &\abs{f(y')}\abs{\Phi_\sigma(|x-y|)-\Phi_\sigma(|x-y'|)} \\
	& \leq \abs{f(y)-f(y')} + \abs{f}_\infty \frac{\abs{y-y'}}{\sigma e^{1/2}}
	\end{align*}
	where in the last inequality we used $\abs{\abs{x-y}-\abs{x-y'}} \leq \abs{y-y'}$ and Lemma \ref{lem:gaussian}. Since $f, \abs{\cdot}$ are both uniformly continuous over $S$ (as continuous functions over a compact surface), $f$ is bounded, \ie, $|f|_\infty<\infty$, equicontinuity of $\set{g_x}$ is proved.
\end{proof}

\begin{lemma}\label{lem:gaussian}
	The gaussian satisfies $\abs{\Phi_\sigma(r')-\Phi_\sigma(r)}\leq \frac{(r'-r)}{\sigma e^{1/2}}$ for $0\leq r < r'$.
\end{lemma}
\begin{proof}
	\begin{align*}
	\abs{\Phi_\sigma(r')-\Phi_\sigma(r)}&\leq \int_r^{r'}\frac{t}{\sigma^2}e^{-\frac{t^2}{2\sigma^2}} dt \leq \frac{(r'-r)}{\sigma e^{1/2}},
	\end{align*}
	where in the last inequality we used the fact that $t e^{-\frac{t^2}{2\sigma^2}}\leq \sigma e^{-1/2}$. 
\end{proof}

Lastly, to justify \eqref{e:practical_choice_of_omega} let us use Theorem \ref{thm:extension_op_properties} and consider $f(x)\equiv 1$,
$$1 \approx \E_S[f](x) \approx c \sum_i \omega_i \Phi(|x-x_i|). $$
Plugging $x=x_{i'}$ we get
$$1 \approx c \sum_i \omega_i \Phi(|x_{i'}-x_i|) = c \tilde{\omega}_{i'}\sum_i \Phi(|x_{i'}-x_i|),$$ where $\tilde{\omega}_{i'}$ is an average of values of $\omega_i$ (note that $\Phi(|x_{i'}-x_i|)$ are fast decaying weights away from $x_{i'}$). Hence,  
$$c\tilde{\omega}_{i'}\approx \frac{1}{\sum_i \Phi(|x_{i'}-x_i|)}.$$

\end{document}